\documentclass{article}
\usepackage{spconf,amsmath,graphicx}

\usepackage{algorithm}
\usepackage[noend]{algpseudocode}
\algrenewcommand\alglinenumber[1]{{\textcolor{gray}{\sf\scriptsize#1}}}
\usepackage{setspace}
\usepackage{amssymb}
\usepackage{amsthm}
\usepackage{mathtools}
\usepackage{url}
\usepackage[dvipsnames]{xcolor}
\usepackage[colorlinks=true,citecolor=blue,urlcolor=black,linkcolor=blue]{hyperref}
\usepackage[utf8]{inputenc}
\usepackage[small]{caption}
\usepackage{graphicx}
\usepackage{amsmath}
\usepackage{booktabs}
\usepackage{nicefrac}
\usepackage{enumitem}
\usepackage{sidecap}
\usepackage{multirow}
\usepackage{marvosym}

\newcommand\simplex{\triangle}
\newcommand{\bs}[1]{\boldsymbol{#1}}
\DeclareMathOperator{\prox}{\mathsf{prox}}
\DeclareMathOperator{\diag}{\mathsf{diag}}
\DeclareMathOperator{\proj}{\mathsf{proj}}
\DeclareMathOperator{\sgn}{\mathsf{sign}}
\DeclareMathOperator*{\softmax}{\mathsf{softmax}}
\DeclareMathOperator*{\sparsemax}{\mathsf{sparsemax}}

\DeclareMathOperator*{\gfm}{\mathsf{gfusedmax}}
\DeclareMathOperator*{\argmin}{arg\,min}

\usepackage{tikz}
\usetikzlibrary{matrix}
\usetikzlibrary{arrows.meta}
\newcommand\ww{w}
\newcommand\xx{z}
\newcommand\pp{p}
\newcommand\x{\bs{\xx}}
\newcommand\p{\bs{\pp}}
\newcommand\w{\bs{\ww}}
\newcommand{\norm}[1]{\|#1\|}
\newcommand\RR{\mathbb{R}}
\newcommand\Rd{\RR^d}
\newcommand\Rk{\RR^k}
\newcommand\id{\iota}
\newcommand*{\diffdchar}{\mathrm{d}}
\newcommand*{\dd}{\mathop{\diffdchar\!}}
\newtheorem{definition}{Definition}
\newtheorem{proposition}{Proposition}

\renewcommand{\paragraph}[1]{{\vspace{\baselineskip}\noindent\normalfont\bfseries#1}}

\def\x{{\mathbf x}}

\title{Sparse and Structured Visual Attention}
\name{Pedro Henrique Martins\textsuperscript{\Neptune} \qquad Vlad Niculae \textsuperscript{\Moon} \qquad Zita Marinho\textsuperscript{\Scorpio\Leo} \qquad Andr\'e F.~T. Martins\textsuperscript{\Neptune\Pluto\Saturn}}
  
  \address{\textsuperscript{\Neptune} Instituto de Telecomunica\c{c}\~{o}es \qquad
      \textsuperscript{\Moon} IvI, University of Amsterdam \qquad \textsuperscript{\Scorpio} Priberam Labs \\
      \textsuperscript{\Leo}Institute of Systems and Robotics \qquad \textsuperscript{\Pluto}LUMLIS (Lisbon ELLIS Unit) \qquad \textsuperscript{\Saturn}Unbabel}
\begin{document}

\begin{table*}
\textcopyright  2021 IEEE. Personal use of this material is permitted. Permission from IEEE must be obtained for all other uses, in any current or future media, including reprinting/republishing this material for advertising or promotional purposes, creating new collective works, for resale or redistribution to servers or lists, or reuse of any copyrighted component of this work in other works. 
\end{table*}

%\ninept
%
\maketitle
\begin{abstract}
Visual attention mechanisms are widely used in multimodal tasks, as visual question answering (VQA). One drawback of softmax-based attention mechanisms is that they assign some probability mass to all image regions, regardless of their adjacency structure and of their relevance to the text.
In this paper, to better link the image structure with the text, we replace the traditional softmax attention mechanism with two alternative sparsity-promoting transformations: 
{\it sparsemax}, which is able to select only the relevant regions (assigning zero weight to the rest), 
and a newly proposed {\it Total-Variation Sparse Attention} (\textsc{TVmax}), which further encourages the joint selection of adjacent spatial locations. 
Experiments in VQA show gains in accuracy as well as higher similarity to human attention, which suggests better interpretability.
\end{abstract}

\section{INTRODUCTION}
\label{sect:introduction}

Vision-language tasks, as visual question answering (VQA), require combining natural language understanding with object and scene recognition. 
While general purpose architectures can be powerful \cite{bahdanau2014neural,vaswani2017attention}, 
the ability to incorporate structural bias is a desirable feature to better link the language and vision components and produce more interpretable decisions. How can we encourage models to look at the relevant objects only, avoiding distractions?

The current state of the art for these tasks is based on deep neural networks with {\bf visual attention} \cite{bottom,tan2019lxmert,yu2019deep,chen2020uniter,jiang2020defense}. These models use attention mechanisms to select either grid features generated by convolutional neural networks (CNNs) pretrained on image recognition datasets \cite{krishna2017visual}, or CNN features of bounding boxes. 
While bounding boxes have the advantage that the attention mechanism can attend to full objects, they require an external object segmentation model, which has a computational cost. 
In this paper, we propose new {\bf selective visual attention mechanisms} over grid features, which owe their ability to select compact objects to the encouragement of joint selection of neighboring regions.

A key component of attention mechanisms is the transformation that maps scores into probability values, with softmax being the standard choice \cite{bahdanau2014neural}. A downside of softmax is that it is {\bf strictly dense}, {\it i.e.}, it devotes some attention probability mass to \emph{every} region in the image. This makes the model less interpretable and, for complex images, it may lead to a ``lack of focus''. 
This is visible in the example of Fig.~\ref{plot_att_vqa}: the model using softmax attends, always, to the entire image and, consequently, wrongly predicts that no one is crossing the bridge.

\begin{figure*}[t]
\begin{center}
\includegraphics[width=13cm]{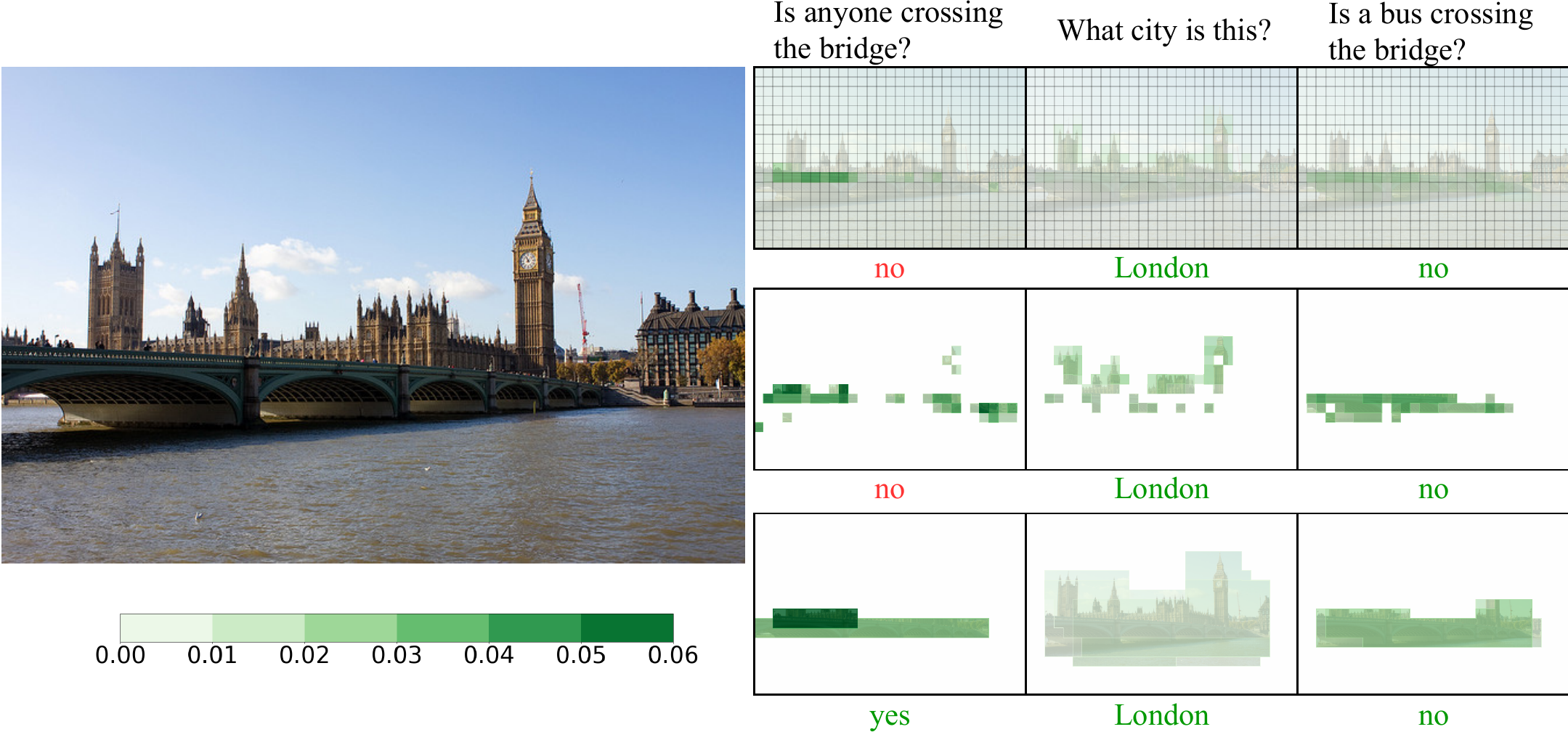}
\end{center}
\caption{VQA example using softmax (top), sparsemax
(middle), and the proposed attention mechanism: \textsc{TVmax} (bottom).}
\label{plot_att_vqa}
\end{figure*}

In this work, we introduce novel selective visual attention mechanisms by endowing them with a new capability: that of {\bf selecting only the relevant features of the image}. To this end, 
we first propose replacing softmax with {\bf sparsemax} \cite{sparsemax}. 
With sparsemax, the attention weights obtained are sparse, leading to the selection (non-zero attention) of only a few relevant features. 
While sparsemax has been applied successfully to NLP to attend over \emph{words} \cite{malaviya2018sparse,sparsemax,peters2019sparse}, its application to attention over \emph{image regions} is so far unexplored. 
However, as can be seen in the example of Fig.~\ref{plot_att_vqa}, despite leading to an increased focus on the relevant features, sparsemax selects discontiguous regions of the image which prevents the model from attending to full objects and reduces interpretability.

Thus, to further encourage the weights of related adjacent spatial locations to be the same
(\emph{e.g.}, parts of an object), we introduce a new attention mechanism: \textbf{Total-Variation Sparse Attention} (which we dub \textsc{TVmax}), inspired by prior work in {\bf structured sparsity} \cite{bach2012optimization,tibshirani2005sparsity}. Two key results of our paper (\S\ref{subsec: gfusedatt} and Propositions~\ref{prop:gfusedmax}--\ref{lemma:groups_}) show that \textsc{TVmax} can be evaluated by composing a proximal operator with a sparsemax projection, and that its Jacobian has a closed-form expression. This leads to an efficient implementation of its forward and backward passes. 

With \textsc{TVmax}, sparsity is allied to the ability of selecting \emph{compact} regions, improving interpretability, as shown in Fig.~\ref{plot_att_vqa}.
Experiments, in VQA, show that \textsc{TVmax} leads to improved accuracy while having attention maps more similar to human attention, suggesting higher interpretability.\footnote{Code available at \url{https://github.com/deep-spin/TVmax}}

\section{SELECTIVE ATTENTION}
\label{sec:att}

Attention mechanisms \cite{bahdanau2014neural} have the ability to dynamically attend to relevant input features, such as
regions of an image. 
To permit end-to-end training with gradient backpropagation, 
they require a differentiable mapping from importance scores $\bs{z} \in \mathbb{R}^k$ to a distribution 
$\bs{p} \in \simplex^k$, where
$
\textstyle \simplex^k \coloneqq \{ \boldsymbol { p } \in \mathbb { R
} ^ { k } \mid \sum^k_{i=1} p_i=1, \bs{p}\geqslant\bs{0}\}
$
 denotes the probability simplex.
The standard choice is the softmax transformation, defined as $[\softmax(\boldsymbol{z})]_i =\frac{\exp(z_i)}{\sum_j \exp(z_j) }.$
Since softmax is strictly positive, its output is {\bf dense}: it always assigns some probability mass to all image regions, even irrelevant ones. 
This accumulation of low probabilities may ``distract'' the model, preventing it from fully attending to the most relevant parts.
This motivates our proposed \textbf{selective}
visual attention mechanisms. 

\subsection{Sparsemax}
To achieve  selective capabilities, we propose the use of  {\bf sparsemax} \cite{sparsemax}, a sparse mapping consisting in
the Euclidean projection of
$\boldsymbol{z}$ onto the  
simplex:
$
\sparsemax( \boldsymbol { z } ) : =
\argmin_{\boldsymbol{p} \in \simplex^k}
\frac{1}{2}
\| \boldsymbol { p } - \boldsymbol { z } \|_2^2.
$
Sparsemax encourages sparse outputs, 
corresponding to the boundary of $\simplex^k$. 
This is an attractive property for visual attention mechanisms, where
often only a few features provide relevant information.

\subsection{{Sparse and Structured Visual Attention}}

Since the model, often, needs to identify the full objects present in the image, the selected regions should be encouraged to have a compact structure.
However, sparsemax is {\bf unstructured} and {\bf index-invariant}, leading it to select discontiguous
regions. To overcome this, we propose a new visual attention mechanism, {\bf \textsc{TVmax}}. \textsc{TVmax} is a (non-trivial) generalization of fusedmax \cite{vlad}, a 1D transformation based on fused lasso, to the 2D case.
For ease of exposition, we first describe how fused lasso is extended to arbitrary graphs, and then we particularize to the 2D case.

\label{subsec: gfused}
Let $\w \in \Rk$ be a vector of weights, and $(V, E)$ be an undirected graph, where $V = \{1, \ldots, k\}$ and $E \subseteq \{(i,j) \in V^2 \mid i < j\}$. 
The generalized fused lasso penalty \cite{tibshirani2005sparsity} is defined as 
$
\Omega_{E}(\w) = \sum_{(i,j) \in E} |\ww_i - \ww_j|.
$
Minimizing $\Omega_E$ encourages ``fused'' solutions, {\it i.e.}, it encourages $\ww_i =
\ww_j$ for $(i,j) \in E$. In particular, its proximal operator 
can be seen as a
\textbf{fused signal approximator}, seeking a vector $\w$ that approximates $\x$
well 
and that is encouraged to be fused:
\begin{equation}\label{eq:prox_fused_}
\prox_{\lambda \Omega_{E}}(\x) = \argmin_{\w \in \Rd} \frac{1}{2}\norm{\w - \x}_2^2 +
\lambda \Omega_E(\w). 
\end{equation}
 {Computing the value} of $\prox_{\lambda \Omega_E}$
is non-trivial in general \cite{gfl}, but for certain edge configurations 
efficient algorithms exist:
\begin{itemize}
    \item  If $E$ forms a chain, 
    the problem is called \textbf{1D total variation} and the penalty is defined as
    $\Omega^{TV}_{1D}(\bs{w}) \coloneqq \sum^{k-1}_{i=1} |w_{i+1}-w_i|.$
    It can be solved in $\mathcal{O}(k)$ time using the \emph{taut string algorithm} \cite{taut,sra}. We use the quasilinear
    algorithm of \cite{condat}, which is very fast in practice.

    \item If the indices are aligned on a 2D grid, as in an image, the problem is called \textbf{2D total variation} and the penalty is defined as $ \Omega^{TV}_{2D}(\bs{W}) \coloneqq \sum_i \Omega^{TV}_{1D}(\bs{w}_{i, :}) + \sum_j \Omega^{TV}_{1D}(\bs{w}_{:, j}),$
    where $\bs{w}_{i, :}$ and $\bs{w}_{:, j}$ denote the rows and columns of $\bs{W}$. Unlike the 1D case, exact algorithms are not available. However, for an input of size $a \times b$, it is possible to \emph{split} the penalty into $a$ column-wise and $b$ row-wise 1D problems, and apply iterative methods, as the proximal Dykstra algorithm~\cite{sra,copt}.
\end{itemize}

\textsc{TVmax} combines 2D total variation (TV2D) regularization with sparsemax. This way, it promotes sparsity and encourages the attention weights of adjacent spatial locations to be the same, selecting contiguous regions of the image. 
\begin{definition}[\textsc{TVmax}]
Let $\bs{z} \in \mathbb{R}^{k}$, such that the indices of $\bs{z}$ can be decomposed into rows and columns. 
The \textsc{TVmax} transformation is defined as 
\begin{equation}\label{eq:tvmax}
\textsc{TVmax}(\bs{z}) \coloneqq
\argmin_{\bs{p} \in \simplex^k}
\dfrac{1}{2} \|\boldsymbol{p}-\boldsymbol{z}\|_2^2
+ \lambda \Omega^{TV}_{2D}(\bs{p}),
\end{equation}
where $\lambda$ is a hyper-parameter controlling the amount of fusion
($\lambda=0$ recovers sparsemax) and $\Omega^{TV}_{2D}$ is the 2D TV penalty.
\end{definition}

\subsection{\textsc{TVmax}'s Forward and Backward Passes}
\label{subsec: gfusedatt}

In order to use the \textsc{TVmax} transformation as a component in a neural network, we need efficient forward and backpropagation algorithms. 
We will now derive these algorithms for a more general case, the 
\textbf{generalized fused sparse attention}.
We follow \cite{vlad} and define
\begin{equation}
\label{eq:gfusedatt}
\operatorname{\mathsf{gfusedmax}}_E(\x) \coloneqq
\argmin_{\p \in \simplex^k} \norm{\p - \x}_2^2 + \lambda \Omega_E(\p).
\end{equation}
This can be seen as a \emph{constrained} fused lasso approximator, because the
solution $\p$ must be a probability distribution vector.
While the optimization function is very similar to Eq.~\ref{eq:prox_fused_},
note the additional constraint $\p \in \simplex^k$.
Fortunately, the following result holds.

\smallskip

\begin{proposition}\label{prop:gfusedmax}
The generalized fusedmax can be expressed as
\label{prop:forward_}
$\operatorname{\mathsf{gfusedmax}}_E(\x) = \proj_{\simplex^k}\left(\prox_{\lambda
\Omega_{E}}(\x)\right).$
\end{proposition} 

\begin{proof}
This result is an extension of Proposition~2 in \cite{vlad}, and also follows
from Corolary~4 of \cite{decomposing}. By taking $f = \id_{\simplex}$,\footnote{$\id_{\simplex}$ is the indicator function of set $\simplex$.} and
noting that $\id_{\simplex}$ is symmetric: if $\p \in \simplex$, then any
vector $\p'$ obtained by permuting $\p$ is also in $\simplex$, because its
values remain non-negative and sum to 1.
\end{proof}

Proposition~\ref{prop:forward_} shows that $\mathsf{gfusedmax}$'s forward pass 
can be computed {simply by composing the proximal step of fused lasso with the forward pass of sparsemax.} 
 It also provides a shortcut for deriving the
Jacobian of $\mathsf{gfusedmax}$ via the \emph{chain rule}.
Denoting by
$  \bs{J}_F$ the Jacobian of $\prox_{\lambda \Omega_E}$, we have:
\begin{equation}
\dfrac{\partial \gfm }{\partial \bs{z}}
= \bs{J}_{\sparsemax}(\prox_{\lambda \Omega_E}(\bs{z})) \bs{J}_{F}(\bs{z}).
\label{jacobian}
\end{equation}
$\bs{J}_{\sparsemax}$ has been derived by \cite{sparsemax}: $\bs{J}_{\sparsemax}(\boldsymbol{z}) = \diag \bs{s} - \frac{1}{\|\bs{s}\|_1} \bs{ss}^\top$, where $s_j = 1$ if $\sparsemax(\xx)_j > 0$ and $s_j = 0$ otherwise.
The next proposition completes the puzzle, giving a full characterization of $\bs{J}_F$.

\smallskip

\begin{proposition}[Group-wise characterization of $\prox_{\lambda \Omega_E}$]
\label{lemma:groups_}
Let $\bs{w}^\star \coloneqq \prox_{\lambda \Omega_E}(\bs{z})$, and denote by $G_i$
the set of indices fused to $\ww_i$ in the solution, defined recursively:
\begin{enumerate}
\item %(i) 
$i \in G_i$ for all $i$, and
\item %(ii) 
$j \in G_i$ $\exists m \in G_i$ such that edge $(m,j) \in E$ and $\ww^\star_m =
\ww^\star_j$. 
\end{enumerate}
Define $s_{ij} = \sgn(\ww^\star_i - \ww^\star_j)$. Then, we have
\begin{equation}\label{eq:groups_}
\ww^\star_i = \frac{1}{|G_i|}
\sum_{j \in G_i}
\left(
\xx_j
+ \sum_{\substack{(m,j) \in E, \\ m \not\in G_i}} \lambda s_{mj}
- \sum_{\substack{(j,m) \in E, \\ m \not\in G_i}} \lambda s_{jm}\right). 
\end{equation}
\end{proposition}
\begin{proof}
 The subgradient optimality conditions of Eq.~\ref{eq:prox_fused_} are \cite{pathwise}:
\begin{equation}\label{eq:subgrad}
\ww^\star_i - \xx_i 
+ \sum_{(i,k) \in E} \lambda t_{ik} 
- \sum_{(k,i) \in E} \lambda t_{ki} 
= 0
\qquad 1 \leq i \leq d.
\end{equation}
where $t_{ij} = s_{ij}$ if $\ww^\star_i \neq
\ww^\star_j$, otherwise $t_{ij}$ is a free variable in $[-1, 1]$.
We focus on a single group $G = G_i$.
Within a fused group, the solution is constant, {\it i.e.}, $\ww^\star_j = \ww$ for $j
\in G$. We separate the sums in Eq.~\ref{eq:subgrad} according to whether $k \in
G$ or not, and move the ``constant'' terms to the right hand side, yielding %the system
\begin{align}\label{eq:subgrad_group}
\ww 
+ \!\!\! \sum_{\substack{(j,k) \in E \\k \in G}} \!\! \lambda t_{jk}
- \!\!\! \sum_{\substack{(k,j) \in E \\k \in G}} \!\! \lambda t_{kj}
%\\\nonumber
= \! \xx_j
- \!\!\! \sum_{\substack{(j, k) \in E \\k \not\in G}} \!\! \lambda s_{jk}
+ \!\!\! \sum_{\substack{(k, j) \in E \\k \not\in G}} \!\! \lambda s_{kj},
%\qquad j \in G. 
\end{align}
for $j \in G$. Summing up
the Eq.~\ref{eq:subgrad_group} over all $j \in G$, we observe that for any edge $(i,j) \in E$ with $i,j \in G$, the term
$\lambda t_{jk}$ appears twice with opposite signs
(as in Eq. 9 in \cite{viallon:hal-00813281}). Thus,
\begin{equation}
\sum_{j \in G} \ww =
\sum_{j \in G} \left(
\xx_j
+ \sum_{\substack{(k,j) \in E \\k \not\in G}} \lambda s_{kj}
- \sum_{\substack{(j,k) \in E \\k \not\in G}} \lambda s_{jk}
\right).
\end{equation}
Dividing by $|G|$ gives exactly Eq.~\ref{eq:groups_}. This reasoning applies to
any group $G_i$.
\end{proof}
Proposition~\ref{lemma:groups_} enables easy computation of a generalized Jacobian of
$\gfm$:
since small perturbations in $\boldsymbol{z}$ never change the groups $G_i$
nor the signs of across-group differences $s_{ij}$, differentiating
Eq.~\ref{eq:groups_} yields
\begin{equation}\label{eq:grad_}
({\bs{J}_F})_{i,j} = \frac{\partial \ww^\star_i}{\partial \xx_j} = \begin{cases}
\frac{1}{|G_i|}, & j \in G_i, \\
0, &j \not \in G_i. \\
\end{cases}
\end{equation}
This generalizes Lemma~1 of \cite{vlad} to arbitrary graphs.

\begin{table*}[t]
\centering \small
\begin{tabular}
{l@{\hspace{3.5ex}}l@{\hspace{4ex}}c@{\hspace{3ex}}c@{\hspace{3ex}}c@{\hspace{3ex}}c@{\hspace{5ex}}c@{\hspace{3ex}}c@{\hspace{3ex}}c@{\hspace{3ex}}c}
\toprule
& & \multicolumn{4}{c}{Test-Dev} & \multicolumn{4}{c}{Test-Standard}\\
\midrule
&  & Y/N & Numb. & Other & Overall & Y/N & Numb. & Other & Overall \\
\midrule
\multirow{2}{*}{bounding boxes} & softmax & 85.14 & 49.59 & \underline{58.72} & 68.57 & 85.56 & 49.54 & \underline{59.11} & 69.04\\
& sparsemax & \underline{85.41} & \underline{50.29} & 58.62 & \underline{68.71} & \underline{85.80} &\underline{50.18} & 59.08 & \underline{69.19}\\
\midrule
\multirow{3}{*}{grid} & softmax & 86.88 & 52.61 & 60.15 & 70.31 & 86.94 & 52.88 & 60.36 & 70.56 \\
& sparsemax & 86.61 & 52.28 & 60.04 & 70.11 & 86.77 & 52.66 & 60.14 & 70.40\\
& \textsc{TVmax} & \textbf{86.92} & \textbf{53.19} & \textbf{60.22} & \textbf{70.42} & \textbf{86.98} & \textbf{53.08} & \textbf{60.56} & \textbf{70.70} \\
\bottomrule
\end{tabular}
\caption{VQA accuracy (per-type and overall) on VQA-2.0 dataset using bounding box features or grid features as input. }
\label{results_vqa}
\end{table*}

\paragraph{Computation. }
\begin{algorithm}[tb]
\small\setstretch{.8}
    \caption{\textsc{TVmax} backward pass}
  \label{alg}
  \begin{algorithmic}
    \State \textbf{Input:} $\p = \textsc{TVmax}(\x)$, $\dd\p \in \RR^k$.
    \State \textbf{Output:} $\dd\x = \bs{J}_{\textsc{TVmax}}^\top (\dd\p) ~ \in \RR^k$
    \State \textbf{Initialize:} $N\leftarrow \varnothing$, $V\leftarrow \varnothing$, $G\leftarrow \varnothing$, $s=0$  
    \State $\dd\w \leftarrow (\bs{J}_{\sparsemax})^\top \dd\p$
    \While{$|\bs{V}| < k$} 
        \State \textbf{pick} $(i_0, j_0) \not\in V$, \textbf{push} $(i_0, j_0)$ \textbf{to} $N$
        \While{$N$ not empty} 
            \State \textbf{pop} $(i, j)$ \textbf{from} $N$
            \If{$\pp_{i,j}=\pp_{i_0, j_0}$} 
                \State $G \leftarrow G \cup \{(i,j)\}, V \leftarrow V \cup \{(i,j)\}$, $s \leftarrow s+ (\dd\w)_{i,j}$ 
                \ForAll{neighbours $(i',j') \sim (i,j)$} 
                    \If{$(i',j') \not\in V$} 
                        \State \textbf{push} $(i',j')$ \textbf{to} $N$ 
                    \EndIf 
                \EndFor
            \EndIf
        \EndWhile
        \If{$\bs{G}$ not empty}
            \State $(\dd\x)_{i,j} \leftarrow \nicefrac{s}{|G|} $ for all $(i,j) \in G$,$G \leftarrow \varnothing$, $s=0$
        \EndIf
    \EndWhile
\end{algorithmic}
\end{algorithm}
As we show in Proposition~\ref{prop:forward_}, computing $\textsc{TVmax}$'s forward pass 
can be done by chaining efficient algorithms for TV2D and sparsemax.
From Eq.~\ref{jacobian} we have that \textsc{TVmax}'s Jacobian can be computed by composing %the Jacobian of the Total Variation proximal operator, 
$\bs{J}_{F}$ and $\bs{J}_{\sparsemax}$. %, which is known.
As derived in Proposition~\ref{lemma:groups_}, 
$(\bs{J}_{F})_{i,j}=\nicefrac{1}{n_{ij}}$ if
$i$ and  $j$ are fused in a group with $n_{ij}$ elements, and $0$ otherwise.
Thus, the backward pass intuitively involves ``spreading'' the credit assigned to one region across all regions fused with it. This can be
implemented by Alg.~\ref{alg} in 
 $\mathcal{O}(N_g\log k)$ where $N_g$ is the number of groups of fused regions. In the worst case, when there are no fused regions, the complexity is $\mathcal{O}(k\log k$).
This algorithm is inspired by flood filling algorithms \cite{flooffill}.

\section{EXPERIMENTS}
To compare the attention mechanisms in VQA, we use the encoder-decoder version of modular co-attention networks \cite{yu2019deep}. 
To represent the image we use grid features pre-trained by \cite{jiang2020defense} on Visual Genome \cite{krishna2017visual} with a ResNet-152 as backbone (``grid'' in Table~\ref{results_vqa}) or bounding box features extracted with Faster R-CNN \cite{ren2015faster} pretrained on Visual Genome (``bounding boxes'' in Table~\ref{results_vqa}). The different attention mechanisms, softmax, sparsemax, and \textsc{TVmax}, are used in the output attention layer. %, to compute the final image representation.
All models were trained on VQA-v2 dataset \cite{vqav2} for $15$ epochs using %the 
Adam %optimizer
\cite{adam} with a learning rate of $\min(2.5t\cdot 10^{-5}$, $1\cdot 10^{-4})$ %for the experiments with
when using 
bounding boxes and $\min(2.5t\cdot 10^{-5}$, $5\cdot 10^{-5})$ for grid features, where $t$ is the epoch number. After 10 epochs, the learning rate is multiplied by $\nicefrac{1}{5}$ every 2 epochs. We set $\lambda=0.01$ for \textsc{TVmax}.
%Results on VQA-v2 dataset \cite{vqav2} are reported in Table~\ref{results_vqa}.%\footnote{Evaluated on 

\paragraph{Results. }
%As can be observed, 
When using bounding box features, sparsemax outperforms softmax, suggesting that a sparse selection of relevant bounding boxes leads to more accurate answers.
When using grid features as input, the model using \textsc{TVmax} attention outperforms all other models. This shows that having sparse attention in conjunction to encouraging the selection of contiguous regions, not only improves interpretability but also leads to an accuracy improvement in VQA. 
Moreover, the superior result of \textsc{TVmax} when compared to sparsemax
corroborates our premise that selecting contiguous regions of the image is beneficial.
We can also see that, as stated in \cite{jiang2020defense}, grid features outperform bounding box features.

\paragraph{Human attention. }
Finally, to understand if \textsc{TVmax} leads to higher interpretability, we compared the attention distributions obtained using the different transformations with human attention. To do so, we used the VQA-HAT dataset \cite{vqahat}, where human attention is obtained by having annotators unblurring the relevant regions of the images. %to answer the questions. 
To compare the attention distributions with the human attention we used the Spearman's rank correlation and the Jensen-Shannon divergence (JS). 
\begin{table}[h]
\centering \small
\begin{tabular}{lcc}
\toprule
& Spearman & JS divergence \\
\midrule
softmax & 0.33 & 0.64  \\
sparsemax & 0.32 & 0.66 \\
\textsc{TVmax} & \textbf{0.37} & \textbf{0.62} \\
\bottomrule
\end{tabular}
\caption{Spearman correlation and JS divergence between attention distributions obtained with the different models and human attention.} 
\label{results_human_attention}
\end{table}
As shown in Table~\ref{results_human_attention}, 
the attention distributions obtained with \textsc{TVmax} are more similar to human attention than %the ones obtained
with softmax and sparsemax.
This indicates that \textsc{TVmax} leads to more interpretable attention distributions. 

\section{CONCLUSIONS}
\label{sect:conclusions}
We propose using sparse and structured visual attention to improve the process of selecting the relevant features. For that, we used sparsemax and introduced \textsc{TVmax}.
By selecting only relevant compact groups of features, \textsc{TVmax} leads to more interpretable attention distributions, as shown by the higher similarity to human attention. Our experiments in VQA show improvements in accuracy when replacing softmax by sparsemax to attend over bounding boxes and when using \textsc{TVmax} to attend over grid features.

\section{ACKNOWLEDGMENTS}
This work was supported by the ERC StG DeepSPIN 758969, by the FCT
through contract UIDB/50008/2020 and contract PD/BD/150633/2020 in the scope of the  Doctoral Program  FCT - PD/00140/2013 NETSyS, and  Lisboa 2020 through ERDF, within project TRAINER (Nº 045347). We thank Pedro M. Q. Aguiar for helpful discussion and feedback.

\bibliographystyle{IEEE}
\bibliography{icip}

\appendix
\onecolumn

\section{EXAMPLES}
\label{examples}
Additional VQA examples, using the softmax, sparsemax, and \textsc{TVmax} attention, are presented in Figures \ref{plot_att_vqa_roses}, \ref{beagles}, \ref{church}, and \ref{ice_cream}.

\begin{figure}[h!]
\begin{center}
\includegraphics[width=16cm]{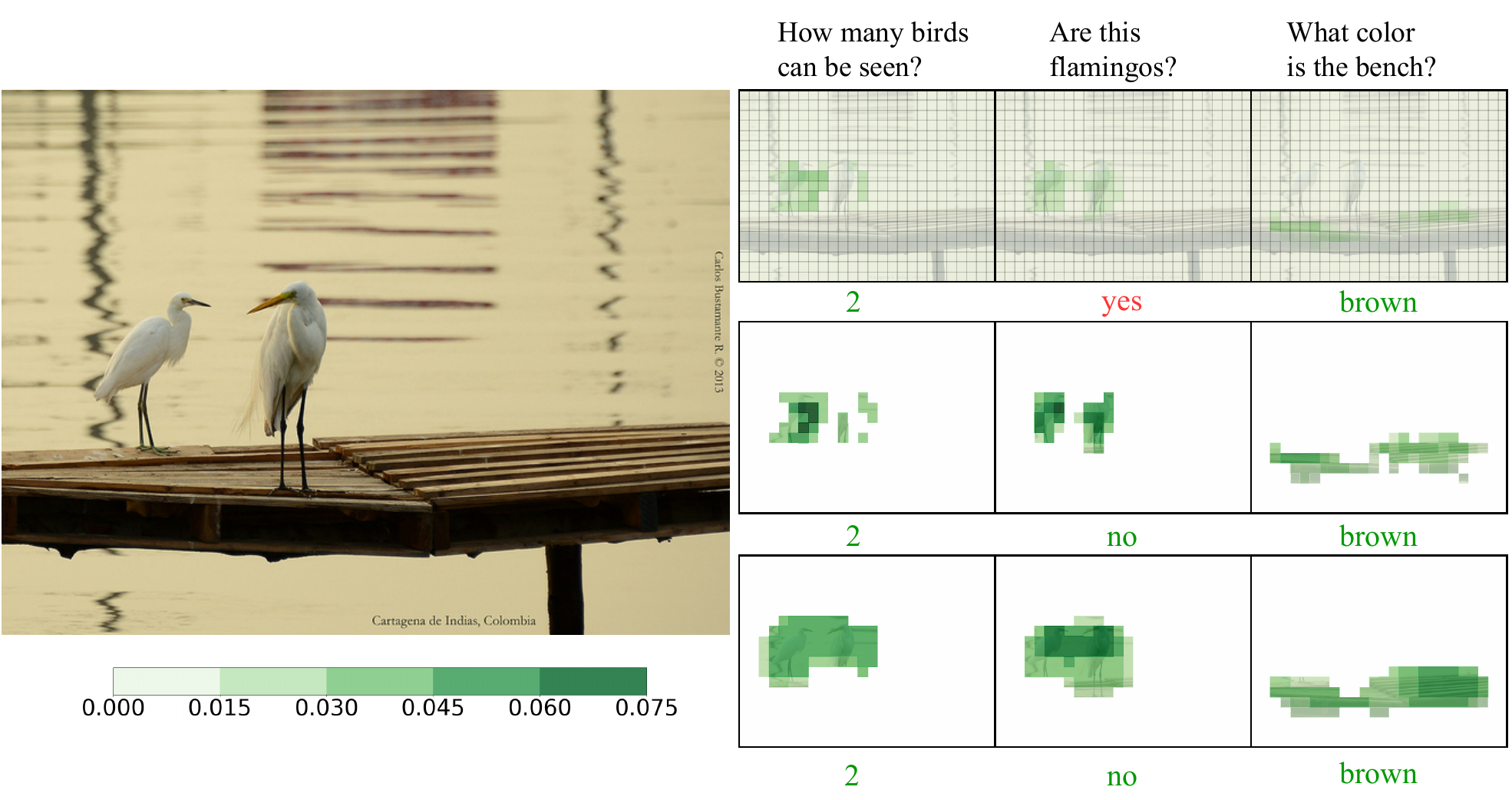}
\end{center}
\caption{VQA using softmax (top), sparsemax
(middle) and \textsc{TVmax} attention (bottom). 
Shading denotes the
attention weight, with white for zero attention.}
\label{plot_att_vqa_roses}
\end{figure}

\begin{figure}[h!]
\begin{center}
\includegraphics[width=16cm]{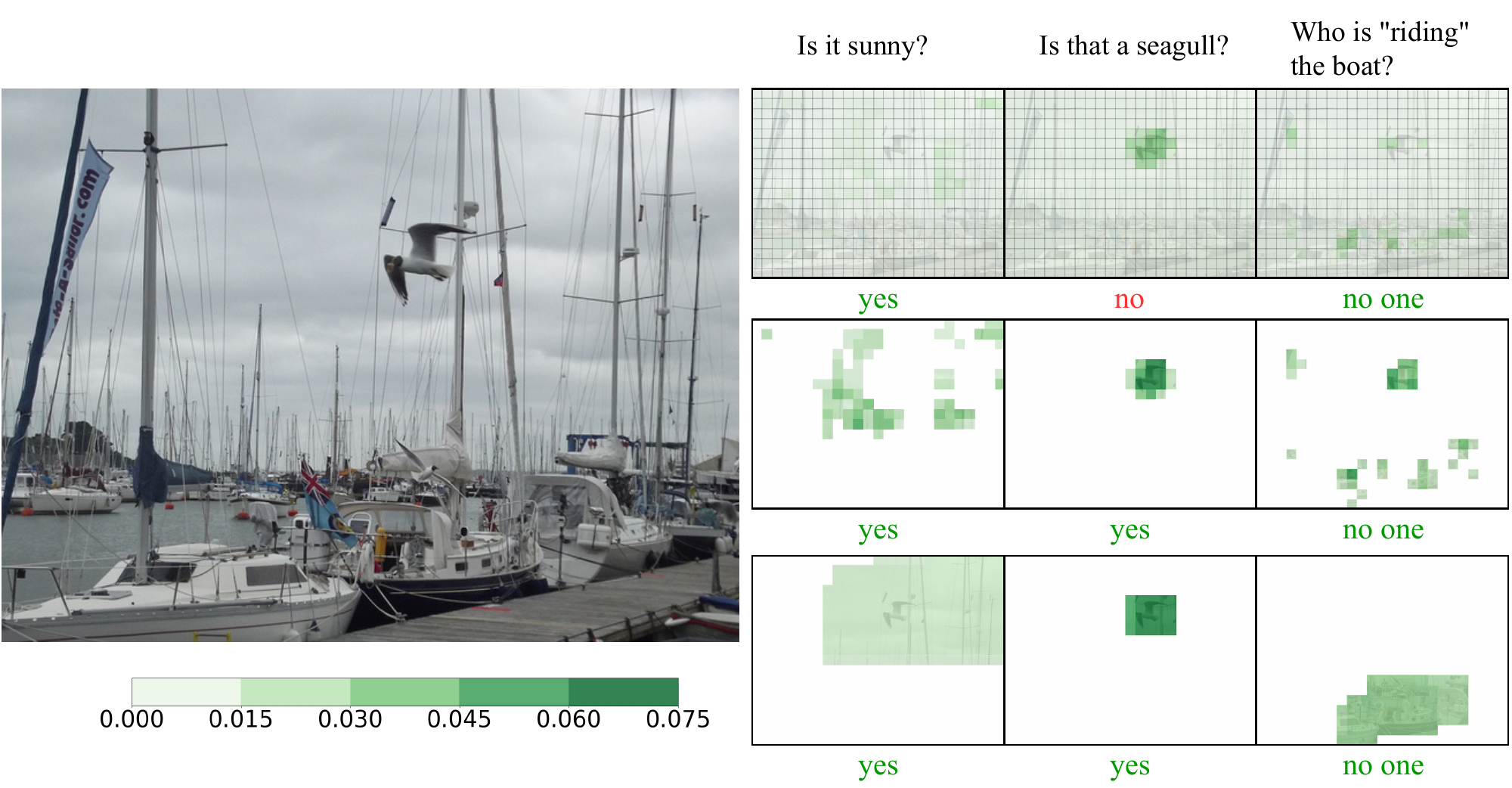}
\end{center}
\caption{VQA using softmax (top), sparsemax
(middle) and \textsc{TVmax} attention (bottom). 
Shading denotes the
attention weight, with white for zero attention.}
\label{beagles}
\end{figure}

\begin{figure}[h!]
\begin{center}
\includegraphics[width=16cm]{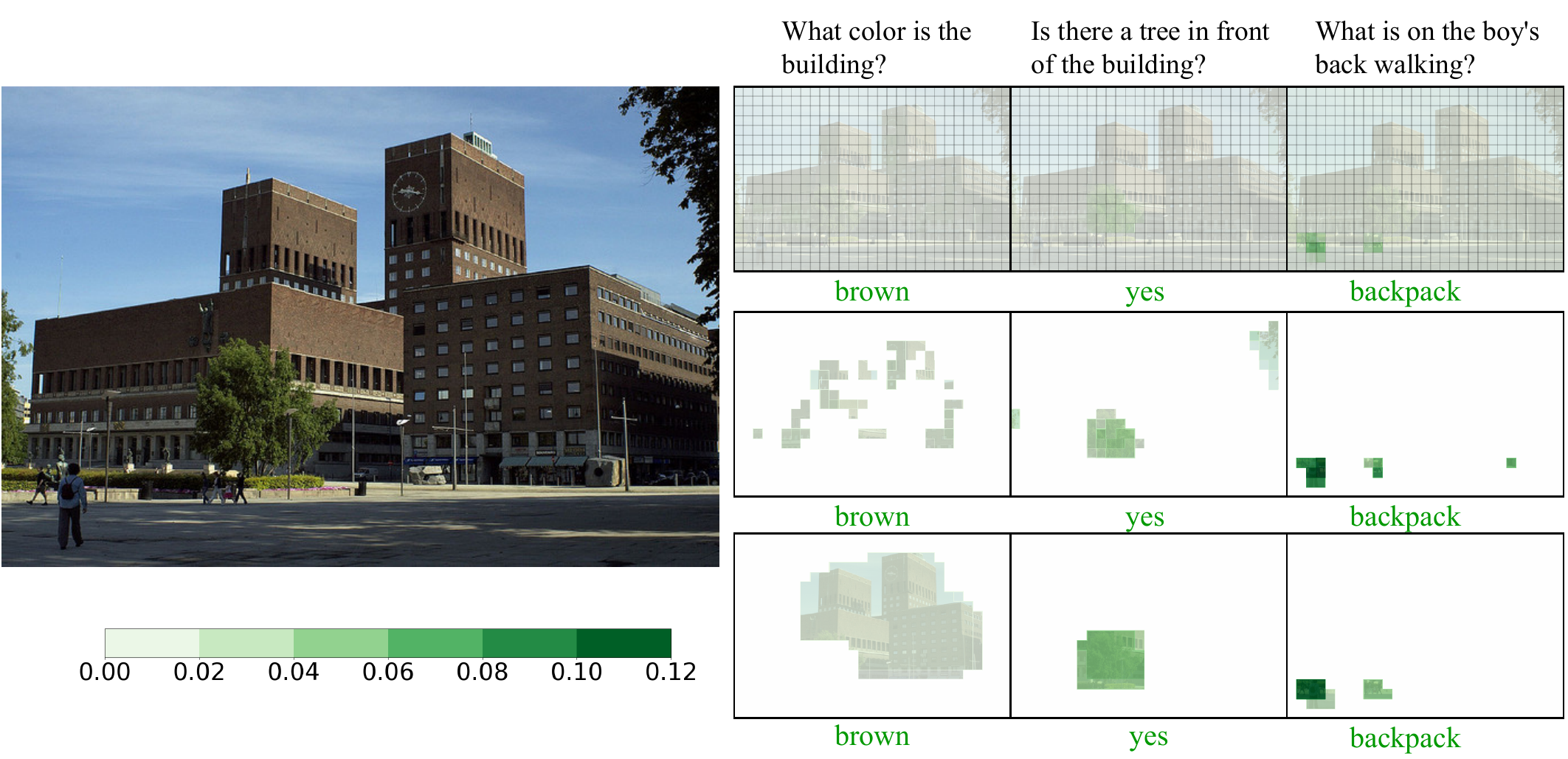}
\end{center}
\caption{VQA using softmax (top), sparsemax
(middle) and \textsc{TVmax} attention (bottom). 
Shading denotes the
attention weight, with white for zero attention.}
\label{church}
\end{figure}

\begin{figure}[h!]
\begin{center}
    \includegraphics[width=16cm]{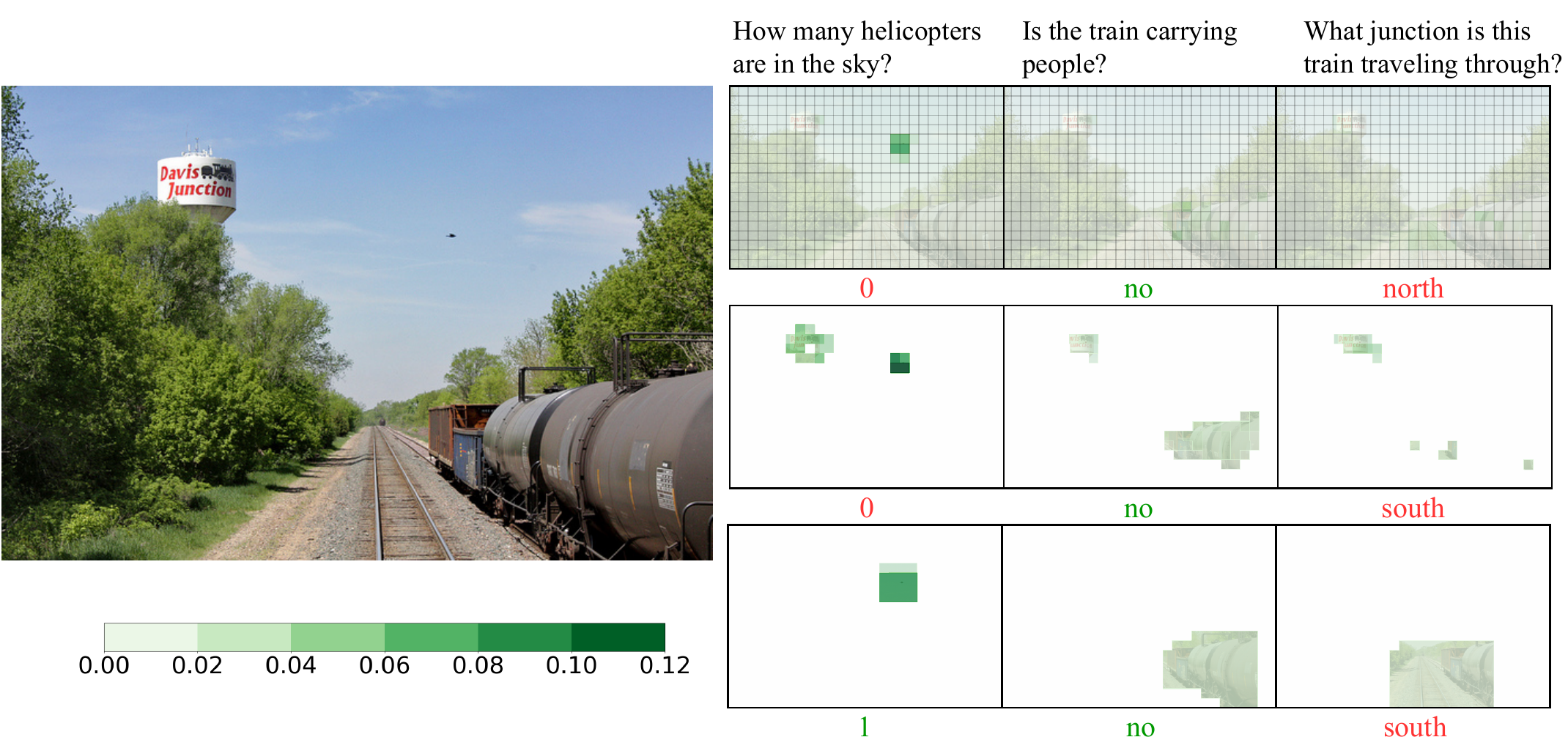}
\end{center}
\caption{VQA using softmax (top), sparsemax
(middle) and \textsc{TVmax} attention (bottom). 
Shading denotes the
attention weight, with white for zero attention.}
\label{ice_cream}
\end{figure}

\section{HUMAN ATTENTION EXAMPLES}
\label{human_attention_examples}
We present in Figure~\ref{human_attention_fig} some images of the VQA-v2 validation set with the corresponding human attention from the VQA-HAT dataset and the attention distributions obtained with the different attention mechanisms: softmax, sparsemax, and \textsc{TVmax}.

\begin{figure}
\begin{center}
    \includegraphics[width=16cm]{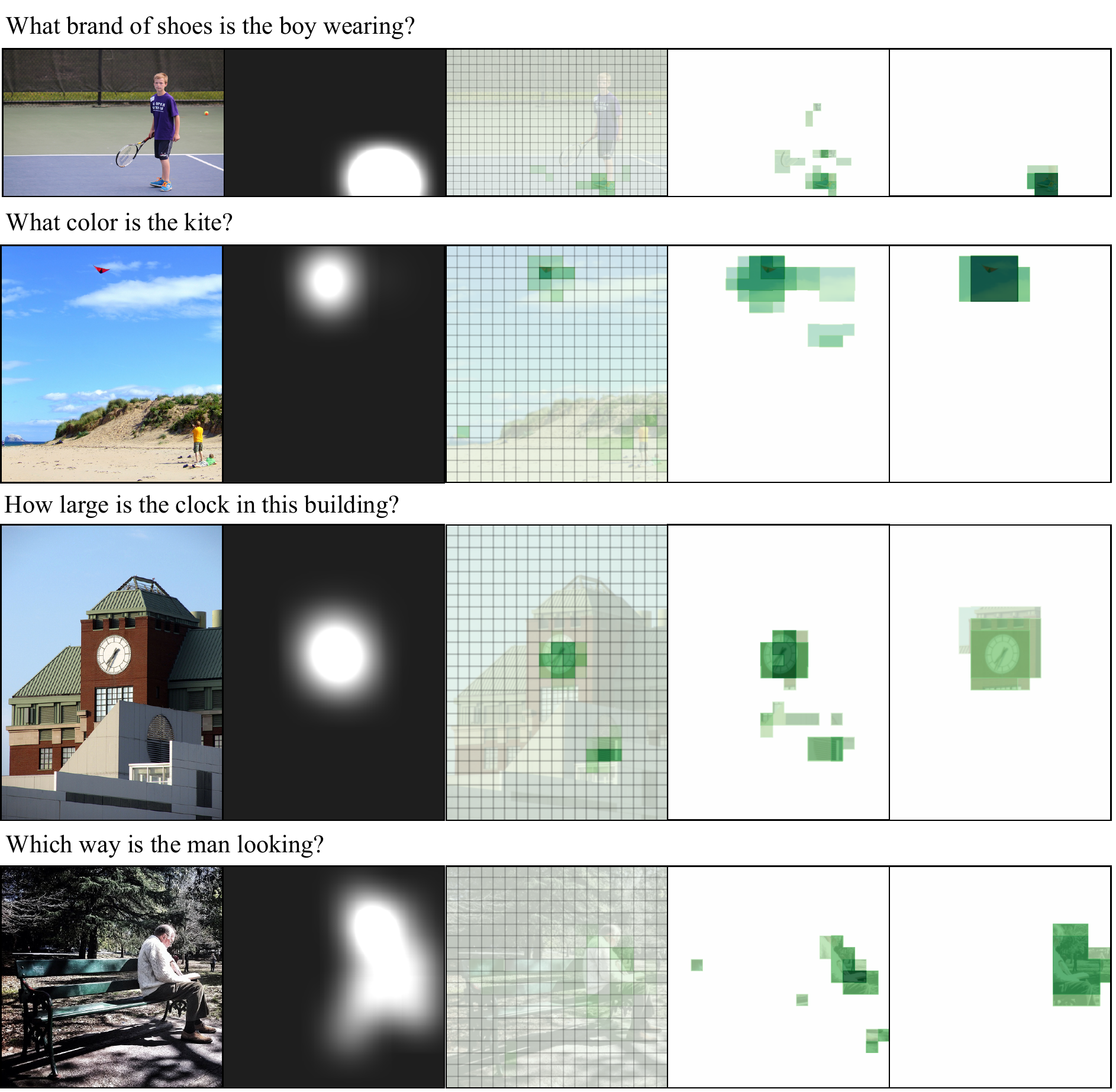}
\end{center}
\caption{Examples of human attention and the attention distributions obtained with the different attention mechanisms. The original image in the left, followed by human attention, softmax attention, sparsemax attention, and \textsc{TVmax} attention.}
\label{human_attention_fig}
\end{figure}

\end{document}